\documentclass{article} 
\usepackage{iclr2018_conference,times}
\usepackage{hyperref}
\usepackage{url}

\usepackage{graphicx}
\usepackage{caption}
\usepackage{subcaption}

\usepackage[utf8]{inputenc}  
\usepackage[T1]{fontenc}  
\usepackage{hyperref}  
\usepackage{booktabs}  
\usepackage{amsfonts}  
\usepackage{nicefrac}  
\usepackage{microtype}  

\usepackage{amsmath}
\usepackage{amsthm}
\usepackage{bm}  

\usepackage{tabularx}
\usepackage{wrapfig}
\usepackage{color}

\newcommand\x{\mathbf{x}}
\newcommand\z{\mathbf{z}}
\newcommand\y{\mathbf{y}}
\renewcommand\v{\mathbf{v}}
\renewcommand\c{\mathbf{c}}

\newtheorem{proposition}{Proposition}
\newtheorem{definition}{Definition}

\title{Challenges in Disentangling Independent Factors of Variation}

\author{\thanks{The authors contributed equally.} Attila Szab\'{o}, \footnotemark[1] Qiyang Hu, Tiziano Portenier, \& Paolo Favaro\\
Institute of Computer Science\\
University of Bern\\
Switzerland\\
\texttt{\{szabo, hu, portenier, zwicker, paolo.favaro\}@inf.unibe.ch}
\And
Matthias Zwicker \\
Institute for Advanced Computer Studies \\
University of Maryland \\
USA\\
\texttt{zwicker@cs.umd.edu}
}

\iclrfinalcopy 

\begin{document}

\maketitle

\begin{abstract}

We study the problem of building models that disentangle independent factors of variation.
Such models could be used to encode features that can efficiently be used for classification and to transfer attributes between different images in image synthesis.
As data we use a weakly labeled training set. Our weak labels indicate what single factor has changed between two data samples, although the relative value of the change is unknown.
This labeling is of particular interest as it may be readily available without annotation costs. To make use of weak labels we introduce an autoencoder model and train it through constraints on image pairs and triplets.
We formally prove that without additional knowledge there is no guarantee that two images with the same factor of variation will be mapped to the same feature. We call this issue the reference ambiguity.
Moreover, we show the role of the feature dimensionality and adversarial training.
We demonstrate experimentally that the proposed model can successfully transfer attributes on several datasets, but show also cases when the reference ambiguity occurs.

\end{abstract}

\section{Introduction}

One way to simplify the problem of classifying or regressing attributes of interest from data is to build an intermediate representation, a feature, where the information about the attributes is better separated than in the input data. Better separation means that some entries of the feature vary only with respect to one and only one attribute. In this way, classifiers and regressors would not need to build invariance to many nuisance attributes. Instead, they could devote more capacity to discriminating the attributes of interest, and possibly achieve better performance. We call this task \emph{disentangling factors of variation}, and we identify attributes with the factors.
In addition to facilitating classification and regression, this task is beneficial to image synthesis. One could build a model to render images, where each input varies only one attribute of the output, and to transfer attributes between images. 


When labeling is possible and available, supervised learning can be used to solve this task.  In general, however, some attributes may not be easily quantifiable (\emph{e.g.}, style). Therefore, we consider using \emph{weak labeling}, where we only know what attribute has changed between two images, although we do not know by how much. This type of labeling may be readily available in many cases without manual annotation. For example, image pairs from a stereo system are automatically labeled with a viewpoint change, albeit unknown. A practical model that can learn from these labels is an encoder-decoder pair subject to a reconstruction constraint. 
In this model the weak labels can be used to define similarities between subsets of the feature obtained from two input images.

In this paper we study the ambiguities in mapping images to factors of variation and the effect of the feature dimensionality on the learned representation. Moreover, we introduce a novel architecture and training procedure to disentangle factors of variation.
We find that the simple reconstruction constraint may fail at disentangling the factors when the dimensionality of the features is too large. We thus introduce an adversarial training to address this problem.
More importantly, in general there is no guarantee that a model will learn to disentangle all factors. We call this challenge the \emph{reference ambiguity} and formally show the conditions under which it appears. In practice, however, we observe experimentally that often the reference ambiguity does not occur on several synthetic datasets.

\section{Related work}


\noindent\textbf{Autoencoders.} Autoencoders in \cite{bourlard1988auto}, \cite{hinton2006reducing}, \cite{bengio2013representation} learn to reconstruct the input data as $\x = \text{Dec}( \text{Enc}(\x) )$, where $\text{Enc}(\x)$ is the internal image representation (the encoder) and $\text{Dec}$ (the decoder) reconstructs the input of the encoder. Variational autoencoders in \cite{kingma2013auto} use a generative model; $p(\x,\z) = p(\x | \z) p(\z)$, where $\x$ is the observed data (images), and $\z$ are latent variables. The encoder estimates the parameters of the posterior, $\text{Enc}(\x) = p(\z | \x)$, and the decoder estimates the conditional likelihood, $\text{Dec}(\z) = p(\x|\z)$.
In \cite{hinton2011transforming} autoencoders are trained with transformed image input pairs. The relative transformation parameters are also fed to the network. Because the internal representation explicitly represents the objects presence and location, the network can learn their absolute position.
One important aspect of the autoencoders is that they encourage latent representations to keep as much information about the input as possible.

\noindent\textbf{GAN.} Generative Adversarial Nets \cite{goodfellow2014generative} learn to sample realistic images with two competing neural networks. The generator $\text{Dec}$ creates images $\x = \text{Dec}(\z)$ from a random noise sample $\z$ and tries to fool a discriminator $\text{Dsc}$, which has to decide whether the image is sampled from the generator $p_g$ or from real images $p_{real}$. After a successful training the discriminator cannot distinguish the real from the generated samples. Adversarial training is often used to enforce constraints on random variables. BIGAN, \cite{donahue2016adversarial} learns a feature representation with adversarial nets by training an encoder $\text{Enc}$, such that $\text{Enc}(\x)$ is Gaussian, when $\x \sim p_{real}$. CoGAN, \cite{liu2016coupled} learns the joint distribution of multi-domain images by having generators and discriminators in each domain, and sharing their weights. They can transform images between domains without being given correspondences.
InfoGan, \cite{chen2016infogan} learns a subset of factors of variation by reproducing parts of the input vector with the discriminator.

\noindent\textbf{Disentangling and independence.} Many recent methods use neural networks for disentangling features, with various degrees of supervision. In \cite{peng2017invariantface} multi-task learning is used with full supervision for pose invariant face recognition. Using both identity and pose labels \cite{Tran_2017_CVPR} can learn pose invariant features and synthesize frontalized faces from any pose.
In \cite{yang2015weakly} autoencoders are used to generate novel viewpoints of objects. They disentangle the object category factor from the viewpoint factor by using as explicit supervision signals: the relative viewpoint transformations between image pairs. In \cite{cheung2014discovering} the output of the encoder is split in two parts: one represents the class label and the other represents the nuisance factors. Their objective function has a penalty term for misclassification and a cross-covariance cost to disentangle class from nuisance factors. Hierarchical Boltzmann Machines are used in \cite{reed2014learning} for disentangling. A subset of hidden units are trained to be sensitive to a specific factor of variation, while being invariant to others. Variational Fair Autoencoders \cite{louizos2015variational} learn a representation that is invariant to specific nuisance factors, while retaining as much information as possible. Autoencoders can also be used for visual analogy \cite{reed2015deep}.
GAN is used for disentangling intrinsic image factors (albedo and normal map) in \cite{Shu_2017_CVPR} without using ground truth labelling. They achieve this by explicitly modeling the physics of the image formation in their network.

The work most related to ours is \cite{mathieu2016disentangling}, where an autoencoder restores an image from another by swapping parts of the internal image representation. Their main improvement over \cite{reed2015deep} is the use of adversarial training, which allows for learning with image pairs instead of image triplets. Therefore, expensive labels like viewpoint alignment between different car types are no longer needed.
One of the differences between this method and ours is that it trains a discriminator for each  of the given labels. A benefit of this approach is the higher selectivity of the discriminator, but a drawback is that the number of model parameters grows linearly with the number of labels. In contrast, we work with image pairs and use a single discriminator so that our method is uninfluenced by the number of labels.
Moreover, we show formally and experimentally the difficulties of disentangling factors of variation.

\section{Disentangling factors of variation}

We are interested in the design and training of two models. One should map a data sample (\emph{e.g.}, an image) to a feature that is explicitly partitioned into subvectors, each associated to a specific factor of variation. The other model should map this feature back to an image. We call the first model the \emph{encoder} and the second model the \emph{decoder}. 
For example, given the image of a car we would like the encoder to yield a feature with two subvectors: one related to the car viewpoint, and the other related to every other car attribute. 
The subvectors of the feature obtained from the encoder should be useful for classification or regression of the corresponding factor that they depend on (the car viewpoint in the example).
This subvector separation would also be very useful to the decoder. In fact, given a valid feature, one could vary only one of its subvectors (for example, the one corresponding to the viewpoint) and observe a variation of the output of the decoder just about its associated factor (the viewpoint). 
Such decoder would enable advanced editing of images. For example, it would allow the transfer of the viewpoint or other car attributes from an image to another. 

The main challenge in the design of these models, when trained on weakly labeled data, is that the factors of variation are latent and introduce ambiguities in the representation. We explain later that avoiding these ambiguities is not possible without using further prior knowledge about the data. We prove this fundamental issue formally, provide an example where it manifests itself and demonstrate it experimentally. 
Interestingly, as the experiments will show, whether the ambiguity emerges or not depends on the complexity of the data.
Next, we introduce our model of the data and formal definitions of our encoder and decoder.

\noindent\textbf{Data model. }
We assume that our observed data $\x$ is generated through some deterministic invertible and smooth process $f$ that depends on the factors $\v\sim p_\v$ and $\c\sim p_\c$, so that $\x = f(\v,\c)$. In our earlier example, $\x$ is an image, $\v$ is a viewpoint (the varying component), $\c$ is a car type (the common component), and $f$ is the rendering engine. We assume that $f$ is unknown, and $\v$ and $\c$ are independent.

\noindent\textbf{The encoder. }
Let $\text{Enc}$ be the encoder mapping images to features. For simplicity, we consider features split into only two column subvectors, $N_\v$ and $N_\c$, one associated to the varying factor $\v$ and the other associated to the common factor $\c$. Then, we have that $\text{Enc}(\x) = [N_\v^\top(\x)~N_\c^\top(\x)]^\top$.
Ideally, we would like to find the inverse of the image formation function, $[N_\v, N_\c] = f^{-1}$, which
separates and recovers the factors $\v$ and $\c$ from data samples $\x$, \emph{i.e.},
\begin{align}
N_\v(f(\v,\c)) = \v \quad\quad N_\c(f(\v,\c)) = \c.
\end{align}
In practice, this is not possible because any bijective transformation of $\v$ and $\c$ could be undone by $f$ and produce the same output $\x$. Therefore, we aim for $N_\v$ and $N_\c$ that satisfy the following \emph{feature disentangling} properties
\begin{align}
R_\v(N_\v(f(\v,\c)))= \v \quad\quad R_\c(N_\c(f(\v,\c)))= \c
\label{eq:featinvariance}
\end{align}
for all $\v$, $\c$, and for some bijective functions $R_\v$ and $R_\c$, so that $N_\v$ is invariant to $\c$ and $N_\c$ is invariant to $\v$.

\noindent\textbf{The decoder.   }
Let $\text{Dec}$ be the decoder mapping features to images. 
A first constraint is that the sequence encoder-decoder forms an \emph{autoencoder}, that is,
\begin{align}
\text{Dec}(N_\v(\x),N_\c(\x)) = \x, \quad\quad \forall \x.
\end{align}
To use the decoder for image synthesis, so that each input subvector affects only one factor in the rendered image, the ideal decoder should satisfy the \emph{data disentangling} property
\begin{align}
\text{Dec}(N_\v(f(\v_1,\c_1)),N_\c(f(\v_2,\c_2))) = f(\v_1,\c_2)
\label{eq:datainvariance}
\end{align}
for any $\v_1$, $\v_2$, $\c_1$, and $\c_2$. The equation above describes the transfer of the varying factor $\v_1$ of $\x_1$ and the common factor $\c_2$ of $\x_2$ to a new image $\x_3 = f(\v_1,\c_2)$.
\\~\\
In the next section, we show that there is an inherent ambiguity in recovering $\v$ from images and in transferring it from one image to another. We show that, given our weakly labeled data, it may not be possible to satisfy all the feature and data disentangling properties. We call this challenge the \emph{reference ambiguity}.

\subsection{The reference ambiguity}

Let us consider the ideal case where we observe the space of all images. 
When weak labels are made available to us, we also know what images $\x_1$ and $\x_2$ share the same $\c$ factor (for example, which images have the same car). This labeling is equivalent to defining the probability density function $p_\c$ and the joint conditional $p_{\x_1,\x_2|\c}$, where
\begin{align}
p_{\x_1,\x_2|\c}(\x_1,\x_2|\c) = \int \delta(\x_1 - f(\v_1,\c))\delta(\x_2 - f(\v_2, \c)) p(\v_1) p(\v_2)d\v_1d\v_2.
\end{align}
Firstly, we show that the labeling allows us to satisfy the feature disentangling property for $\c$ ~\eqref{eq:featinvariance}. For any $[\x_1,\x_2]\sim p_{\x_1,\x_2|\c}$ we impose $N_\c(\x_1)=N_\c(\x_2)$.
In particular, this equation is true for pairs when one of the two images is held fixed. Thus, $N_\c(\x_1) = C(\c)$, where the function $C$ only depends on $\c$, as $N_\c$ is invariant to $\v$.
Lastly, images with the same varying factor, but different common factor must also result in different features, $C(\c_1)  = N_\v(f(\v, \c_1)) \ne N_\v(\v, \c_2) = C(\c_2)$, otherwise the autoencoder constraint cannot be satisfied. Then, there exists a bijective function $R_\c = C^{-1}$ such that property~\eqref{eq:featinvariance} is satisfied for $\c$.
Unfortunately, this is not true in general for the other disentangling properties. 


\begin{definition}
A function $g$ reproduces the data distribution, when it generates samples $\y_1 = g(\v_1,\c)$ and $\y_2 = g(\v_2,\c)$ that have the same distribution as the data. Formally, $[\y_1, \y_2] \sim p_{\x_1, \x_2}$, where the latent factors are independent, $\v_1 \sim p_\v$, $\v_2 \sim p_\v$ and $\c \sim p_\c$.
\end{definition}
The reference ambiguity occurs, when a decoder reproduces the data without satisfying the disentangling properties.

\begin{proposition}
Let $p_\v$ assign the same probability value to at least two different instances of $\v$. Then, we can find encoders that reproduce the data distribution, but do not satisfy the disentangling properties for $\v$ in ~\eqref{eq:featinvariance} and \eqref{eq:datainvariance}.
\end{proposition}
\begin{proof}
We already saw that $N_\c$ satisfies ~\eqref{eq:featinvariance}, so we can choose $N_\c = f_\c^{-1}$, the ideal encoding.
Now we look at defining $N_\v$ and the decoder.
The iso-probability property of $p_\v$ implies that there exists a mapping $T(\v,\c)$, such that $T(\v,\c)\sim p_\v$ and $T(\v,\c_1)\neq T(\v,\c_2)$ for some $\v$ and $\c_1\neq \c_2$. For example, let us denote with $\v_1\neq \v_2$ two varying components such that $p_\v(\v_1) = p_\v(\v_2)$. Then, let 
\begin{align}
T(\v,\c) \doteq \begin{cases}
\v   & \text{if } \v \neq \v_1,\v_2\\
\v_1 & \text{if } \v = \v_1 \lor \v_2 \text{ and } \c \in {\cal C}  \\ 
\v_2 & \text{if } \v = \v_1 \vee \v_2 \text{ and } \c \notin {\cal C} 
\end{cases}
\end{align}
and ${\cal C}$ is a subset of the domain of $\c$, where $\int_{\cal C}p_\c(\c)d\c=1/2$.
Now, let us define the encoder as $N_\v(f(\v,\c)) = T(\v,\c)$. By using the autoencoder constraint, the decoder satisfies
\begin{align}
\text{Dec}(N_\v(f(\v,\c)),N_\c(f(\v,\c))) = \text{Dec}(T(\v,\c),\c) = f(\v,\c).
\end{align}
Because $T(\v,\c)\sim p_\v$ and $\c \sim p_\c$ by construction, and $T(\v,\c)$ and $\c$ are independent, our encoder-decoder pair defines a data distribution identical to that given as training set
\begin{align}
[\text{Dec}(T(\v_1,\c),\c),\text{Dec}(T(\v_2,\c),\c) ]\sim p_{\x_1,\x_2}.
\end{align}
The feature disentanglement property is not satisfied because $N_\v(f(\v_1, \c_1)) = T(\v_1,\c_1) \neq  T(\v_1,\c_2) = N_\v(f(\v_1, \c_2))$, when $\c_1 \in {\cal C}$ and $\c_2 \not\in {\cal C}$.
Similarly, the data disentanglement property does not hold, because $\text{Dec}(T(\v_1,\c_1), \c_1) \neq \text{Dec}(T(\v_1,\c_2), \c_2 )$.
\end{proof}

The above proposition implies that we cannot learn to disentangle all the factors of variation from weakly labeled data, even if we had access to all the data and knew the distributions $p_\v$ and $p_\c$.

To better understand it, let us consider a practical example. Let $\v\sim {\cal U}[ -\pi, \pi ]$ be the (continuous) viewpoint (the azimuth angle) and $\c \sim {\cal B}(0.5)$ the car type, where $\cal U$ denotes the uniform distribution and ${\cal B}(0.5)$ the Bernoulli distribution with probability $p_\c(\c=0)=p_\c(\c=1)=0.5$ (\emph{i.e.}, there are only $2$ car types). In this case, every instance of $\v$ is iso-probable in $p_\v$ so we have the worst scenario for the reference ambiguity. We can define the function $T(\v,\c) = \v (2 \c -1)$ so that the mapping of $\v$ is mirrored as we change the car type. By definition $T(\v,\c)\sim {\cal U}[ -\pi, \pi ]$ for any $\c$ and $T(\v,\c_1)\neq T(\v,\c_2)$ for $\v\neq 0$ and $\c_1\neq \c_2$. So we cannot tell the difference between $T$ and the ideal correct mapping to the viewpoint factor.
This is equivalent to an encoder $N_\v(f(\v,\c)) = T(\v,\c)$ that reverses the ordering of the azimuth of car $1$ with respect to car $0$. Each car has its own reference system, and thus it is not possible to transfer the viewpoint from one system to the other.

We now introduce a training procedure to build the encoder and decoder from weakly labeled data. 
We will use these models to illustrate several challenges: 1) the reference ambiguity, 2) the choice of the feature dimensionality and 3) the normalization layers (see the Implementation section).

\subsection{Model training}
In our training procedure we use two terms in the objective function: an \emph{autoencoder loss} and an \emph{adversarial loss}. We describe these losses in functional form, however the components are implemented using neural networks. In all our terms we use the following sampling of independent factors
\begin{align}
\c_1,\c_3 \sim p_\c, \quad \v_1,\v_2,\v_3 \sim p_\v.
\end{align}
The images are formed as $\x_1 = f(\v_1,\c_1)$, $\x_2 = f(\v_2,\c_1)$ and $\x_3 = f(\v_3,\c_3)$. The images $\x_1$ and $\x_2$ share the same common factor, and $\x_1$ and $\x_3$ are independent. In our objective functions, we use either pairs or triplets of the above images. We denote the inverse of the rendering engine as $f^{-1}=[ f_\v^{-1}, f_\c^{-1}]$, where the subscript refers to the recovered factor.

\noindent\textbf{Autoencoder loss. } In this term, we use images $\x_1$ and $\x_2$ with the same common factor $\c_1$. We feed both images to the encoder. Since both images share the same $\c_1$, we impose that the decoder should reconstruct $\x_1$ from the encoder subvector $N_\v(\x_1)$ and the encoder subvector $N_\c(\x_2)$, and similarly for the reconstruction of $\x_2$. The autoencoder objective is thus defined as
\begin{align}
{\cal L}_{AE} \doteq E_{\x_1,\x_2}\Big[& \big|\x_1 - \text{Dec}(N_\v(\x_1), N_\c(\x_2)) \big|^2
+ \big|\x_2 - \text{Dec}(N_\v(\x_2), N_\c(\x_1)) \big|^2  \Big].
\end{align}

\noindent\textbf{Adversarial loss. } 
We introduce an adversarial training where the \emph{generator} is our encoder-decoder pair and the \emph{discriminator} Dsc is a neural network, which takes image pairs as input. The discriminator learns to distinguish between real image pairs $[\x_1, \x_2]$ and fake ones $[ \x_1, \x_{3\oplus 1}]$, where $\x_{3\oplus 1} \doteq \text{Dec}(N_\v(\x_3), N_\c(\x_1) )]$. If the encoder were ideal, the image $\x_{3\oplus 1}$ would be the result of taking the common component from $\x_1$ and the viewpoint component from $\x_3$. The generator learns to fool the discriminator, so that $\x_{3\oplus 1}$ looks like the random variable $\x_2$ (the common component is $\c_1$ and the varying component is independent of $\v_1$). To this purpose, the decoder must make use of $N_\c(\x_1)$, since $\x_3$ does not carry any information about $\c_1$.
The objective function is thus defined as
\begin{align}
{\cal L}_{GAN} \doteq
E_{\x_1,\x_2} \Big[ \log(\text{Dsc}(\x_1,\x_2)) \Big] +
E_{\x_1,\x_3} \Big[ \log(1-\text{Dsc}( \x_1, \x_{3\oplus 1} ) )  \Big].
\end{align}

\noindent\textbf{Composite loss. } 
Finally, we optimize the weighted sum of the two losses ${\cal L} = {\cal L}_{AE} + \lambda {\cal L}_{GAN}$, 
\begin{align}
\min_{\text{Dec},\text{Enc}} \max_{\text{Dsc}} {\cal L}_{AE}(\text{Dec},\text{Enc}) +
\lambda {\cal L}_{GAN}(\text{Dec},\text{Enc},\text{Dsc})\nonumber
\end{align}
where $\lambda$ regulates the relative importance of the two losses.

\noindent\textbf{Shortcut problem. }
Ideally, at the global minimum of ${\cal L}_{AE}$, $N_\v$ relates only to the factor $\v$ and $N_\c$ only to $\c$. However, the encoder may map a complete description of its input into $N_\v$ and the decoder may completely ignore $N_\c$. We call this challenge the \emph{shortcut problem}. When the shortcut problem occurs, the decoder is invariant to its second output, so it does not transfer the $\c$ factor correctly,
\begin{eqnarray}
\text{Dec}(N_\v(\x_3), N_\c(\x_1)) = \x_3.
\end{eqnarray}
The shortcut problem can be addressed by reducing the dimensionality of $N_\v$, so that it cannot build a complete representation of all input images. This also forces the encoder to make use of $N_\c$ for the common factor. However, this strategy may not be convenient as it leads to a time consuming trial-and-error procedure to find the correct dimensionality. A better way to address the shortcut problem is to use adversarial training. For our analysis we assume that the discriminator is perfect and the global optimum of the adversarial training has been reached. Thus, the real and fake image pair distributions are identical, and any statistics of the two distributions should also match. We compute statistics of the inverse of the common component $f_\c^{-1}$. For the images $\x_1$ and $\x_2$ we obtain
\begin{align}
E_{\x_1,\x_2} \Big[ |f_\c^{-1}(\x_1) - f_\c^{-1}(\x_2) |^2 \Big] = E_{\c_1} \Big[ |\c_1 - \c_1|^2\Big] = 0
\end{align}
by construction (of $\x_1$ and $\x_2$).
For the images $\x_1$ and $\x_{3\oplus1}$ we obtain
\begin{align}
E_{\x_1 , \x_3 } \Big[ |f_\c^{-1}(\x_1) - f_\c^{-1}(\x_{3\oplus1}) |^2 \Big] =
E_{\v_1, \c_1 , \v_3 , \c_3 } \Big[ | \c_1 - \c_{3\oplus1} |^2\Big] \ge 0,
\end{align}
where $\c_{3\oplus1} = f_\c^{-1}(\x_{3\oplus1})$. We achieve equality if and only if $\c_1 = \c_{3\oplus1}$ everywhere. This means that the decoder must use $N_\c$ to recover the common component $\c$ of its input and $N_\c$ is sufficient to recover it.

\subsection{Implementation}

In our implementation we use convolutional neural networks for all the models. We denote with $\theta$ the parameters associated to each network. Then, the optimization of the composite loss can be written as
\begin{align}
\hat\theta_\text{Dec},\hat\theta_\text{Enc},\hat\theta_\text{Dsc} = \arg \min_{\theta_\text{Dec},\theta_\text{Enc}} \max_{\theta_\text{Dsc}} {\cal L}(\theta_\text{Dec},\theta_\text{Enc},\theta_\text{Dsc}).
\end{align}
We choose $\lambda = 1$ and also add regularization to the adversarial loss so that each logarithm has a minimum value. We define $\log_\epsilon  \text{Dsc}(\x_1,\x_2) = \log ( \epsilon + \text{Dsc}(\x_1,\x_2) )$ (and similarly for the other logarithmic term) and use $\epsilon = 10^{-12}$.
The main components of our neural network are shown in Fig.~\ref{fig:network}. The architecture of the encoder and the decoder were taken from DCGAN \cite{radford2015dcGAN}, with slight modifications. We added fully connected layers at the output of the encoder and to the input of the decoder. For the discriminator we used a simplified version of the VGG \cite{simonyan2014very} network. As the input to the discriminator is an image pair, we concatenate them along the color channels.

\begin{figure}[h]
	\centering
	\includegraphics[width=0.5\linewidth]{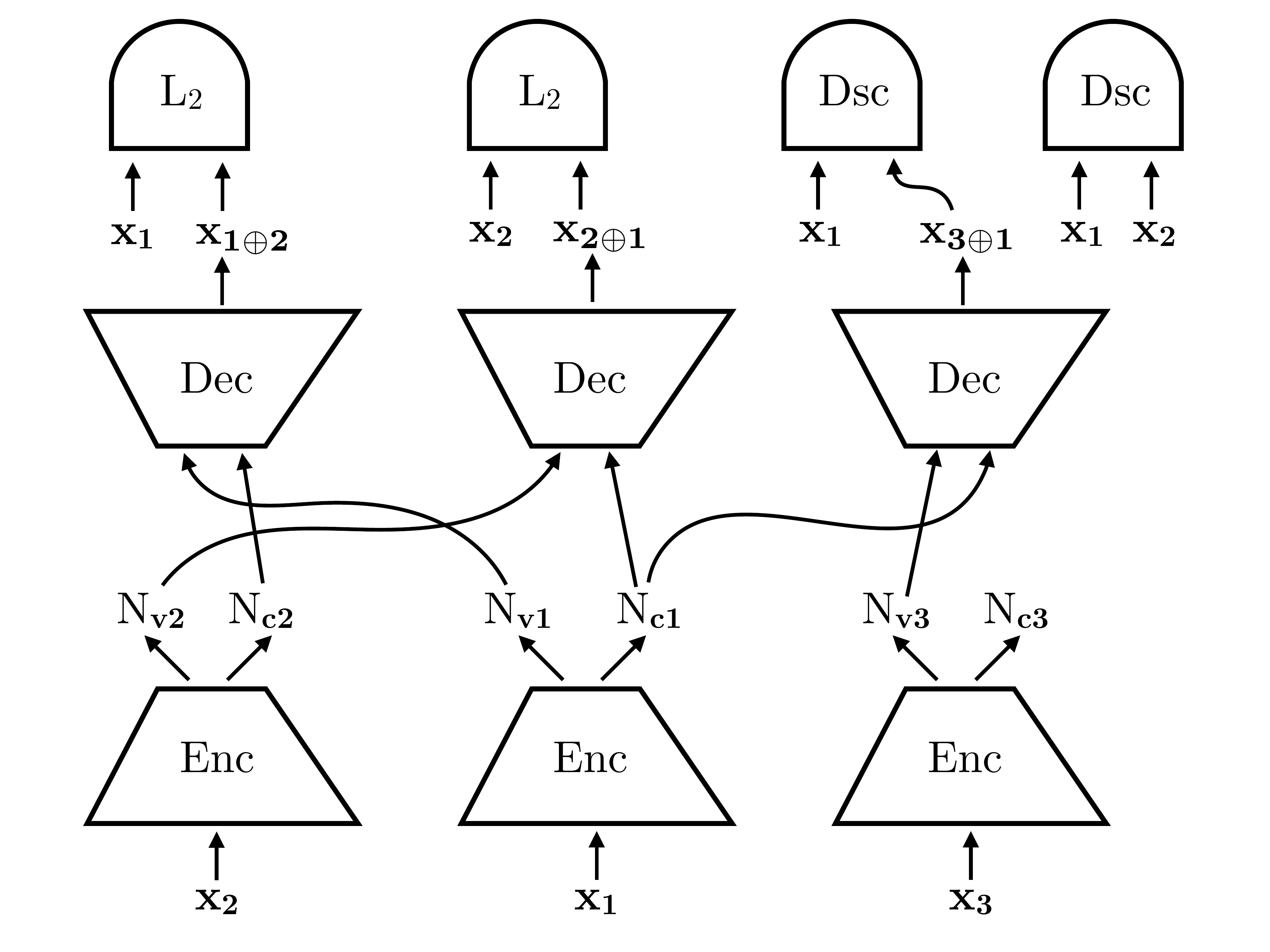}
	\caption{Learning to disentangle factors of variation. The scheme above shows how the encoder (Enc), the decoder (Dec) and the discriminator (Dsc) are trained with input triplets. The components with the same name share weights.}
	\label{fig:network}
\end{figure}

\noindent\textbf{Normalization. }
In our architecture both the encoder and the decoder networks use blocks with a convolutional layer, a nonlinear activation function (ReLU/leaky ReLU) and a normalization layer, typically, batch normalization (BN).
As an alternative to BN we consider the recently introduced \emph{instance normalization} (IN) \cite{UlyanovVL17}. The main difference between BN and IN is that the latter just computes the mean and standard deviation across the spatial domain of the input and not along the batch dimension. Thus, the shift and scaling for the output of each layer is the same at every iteration for the same input image. In practice, we find that IN improves the performance.

\section{Experiments}

We tested our method on the MNIST, Sprites and ShapeNet datasets. We performed qualitative experiments on attribute transfer, and quantitative tests on the nearest neighbor classification task. We show results with models using only the autoencoder loss (\textbf{AE}) and the composite loss (\textbf{AE+GAN}).

\noindent\textbf{MNIST.  }
The MNIST dataset \cite{MNIST} contains handwritten grayscale digits of size $28 \times 28$ pixel. There are $60$K images of $10$ classes for training and $10$K for testing. The common factor is the digit class and the varying factor is the intraclass variation. We take image pairs that have the same digit for training, and use our full model \textbf{AE+GAN} with dimensions $64$ for $N_\v$ and $64$ for $N_\c$.
In Fig. ~\ref{fig:mnistsprites} (a) and (b) we show the transfer of varying factors. Qualitatively, both our method and \cite{mathieu2016disentangling} perform well. We observe neither the reference ambiguity nor the shortcut problem in this case.

\begin{figure*}[h]
	\centering
	\begin{subfigure}[b]{.4\textwidth}
		\includegraphics[width=1\linewidth]{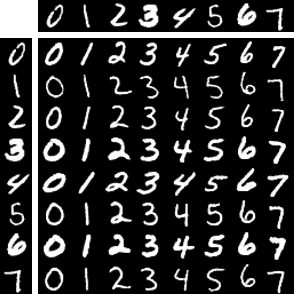}
		\caption{}
	\end{subfigure}
	\hspace{0.07cm}
	\begin{subfigure}[b]{.4\textwidth}
		\includegraphics[width=1\linewidth]{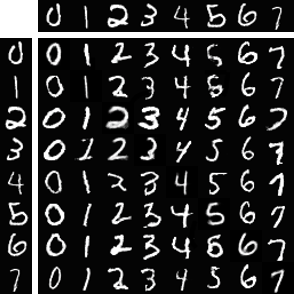}
		\caption{}
	\end{subfigure}
	\hspace{0.07cm}
	\begin{subfigure}[b]{.4\textwidth}
		\includegraphics[width=1\linewidth]{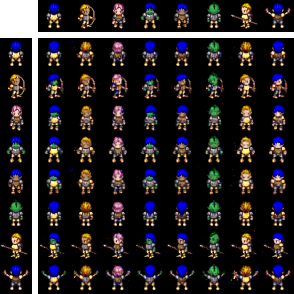}
		\caption{}
	\end{subfigure}
	\hspace{0.07cm}
	\begin{subfigure}[b]{.4\textwidth}
		\includegraphics[width=1\linewidth]{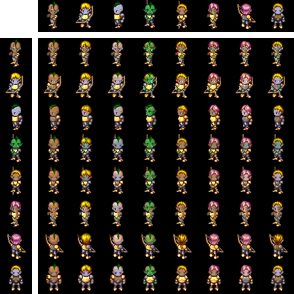}
		\caption{}
	\end{subfigure}
	\caption{Renderings of transferred features. In all figures the variable factor is transferred from the left column and the common factor from the top row. {(a)} MNIST \protect\cite{mathieu2016disentangling}; {(b)} MNIST (ours); {(c)} Sprites \protect\cite{mathieu2016disentangling}; {(d)} Sprites (ours).}
	\label{fig:mnistsprites}
\end{figure*}

\noindent\textbf{Sprites.  }
The Sprites dataset \cite{reed2015deep} contains $60 \time 60$ pixel color images of animated characters (sprites). There are $672$ sprites, $500$ for training, $100$ for testing and $72$ for validation. Each sprite has $20$ animations and $178$ images, so the full dataset has $120$K images in total. There are many changes in the appearance of the sprites, they differ in their body shape, gender, hair, armor, arm type, greaves, and weapon. We consider character identity as the common factor and the pose as the varying factor. We train our system using image pairs of the same sprite and do not exploit labels on their pose. We train the \textbf{AE+GAN} model with dimensions $64$ for $N_\v$ and $448$ for $N_\c$.
Fig. ~\ref{fig:mnistsprites} (c) and (d) show results on the attribute transfer task. Both our method and \cite{mathieu2016disentangling}'s transfer the identity of the sprites correctly.

\noindent\textbf{ShapeNet with a white background.  }
The ShapeNet dataset \cite{shapenet2015} contains 3D objects than we can render from different viewpoints.
We consider only one category (cars) for a set of fixed viewpoints. Cars have high intraclass variability and they do not have rotational symmetries. We used approximately $3$K car types for training and $300$ for testing. We rendered $24$ possible viewpoints around each object in a full circle, resulting in $80$K images in total. The elevation was fixed to $15$ degrees and azimuth angles were spaced $15$ degrees apart. We normalized the size of the objects to fit in a $100 \times 100$ pixel bounding box, and placed it in the middle of a $128 \times 128$ pixel image.
We trained both \textbf{AE} and \textbf{AE+GAN} on ShapeNet, and tried different settings for the feature dimensions $N_\v$. The size of the common feature $N_\c$ was fixed to $1024$ dimensions.
Fig.~\ref{fig:shapenet} shows the attribute transfer on the Shapenet dataset with a white background. We compare the methods \textbf{AE} and \textbf{AE+GAN} with different feature dimension of $N_\v$. We can observe that the transferring performance degrades for \textbf{AE}, when we increase the feature size of $N_\v$. As expected, the autoencoder takes the shortcut and tries to store all information into $N_\v$. The model \textbf{AE+GAN} instead renders images without loss of quality, independently of the feature dimension. Furthermore, none of the models exhibits the reference ambiguity: In all cases the viewpoint could be transferred correctly.

\begin{figure*}[h]
	\centering
	\begin{subfigure}[b]{.45\textwidth}
		\includegraphics[width=1\linewidth]{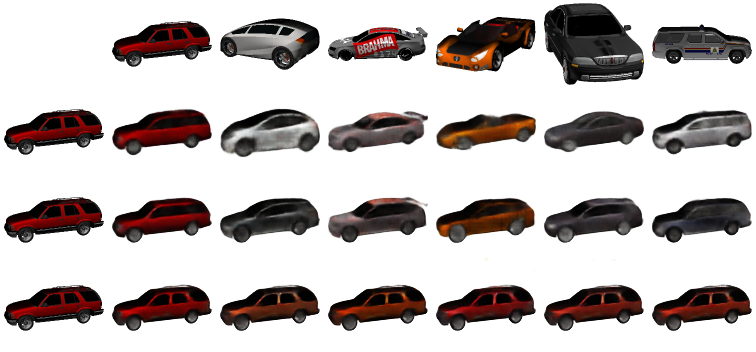}
		\caption{}
	\end{subfigure}\hspace{.05\textwidth}
	\begin{subfigure}[b]{.45\textwidth}
		\includegraphics[width=1\linewidth]{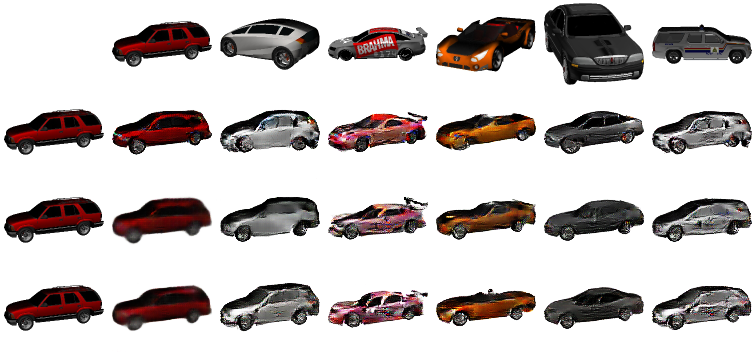}
		\caption{}
	\end{subfigure}
	\caption{Feature transfer on Shapenet. {(a)} synthesized images with \textbf{AE}, where the top row shows images from which the car type is taken. The second, third and fourth row show the decoder renderings using $2$, $16$ and $128$ dimensions for the feature $N_\v$. {(b)} images synthesized with \textbf{AE+GAN}. The setting for the inputs and feature dimensions are the same as in (a).}
	\label{fig:shapenet}
\end{figure*}

\begin{figure}[h]
	\centering
	\begin{subfigure}[b]{.25\textwidth}
		\includegraphics[width=1\linewidth]{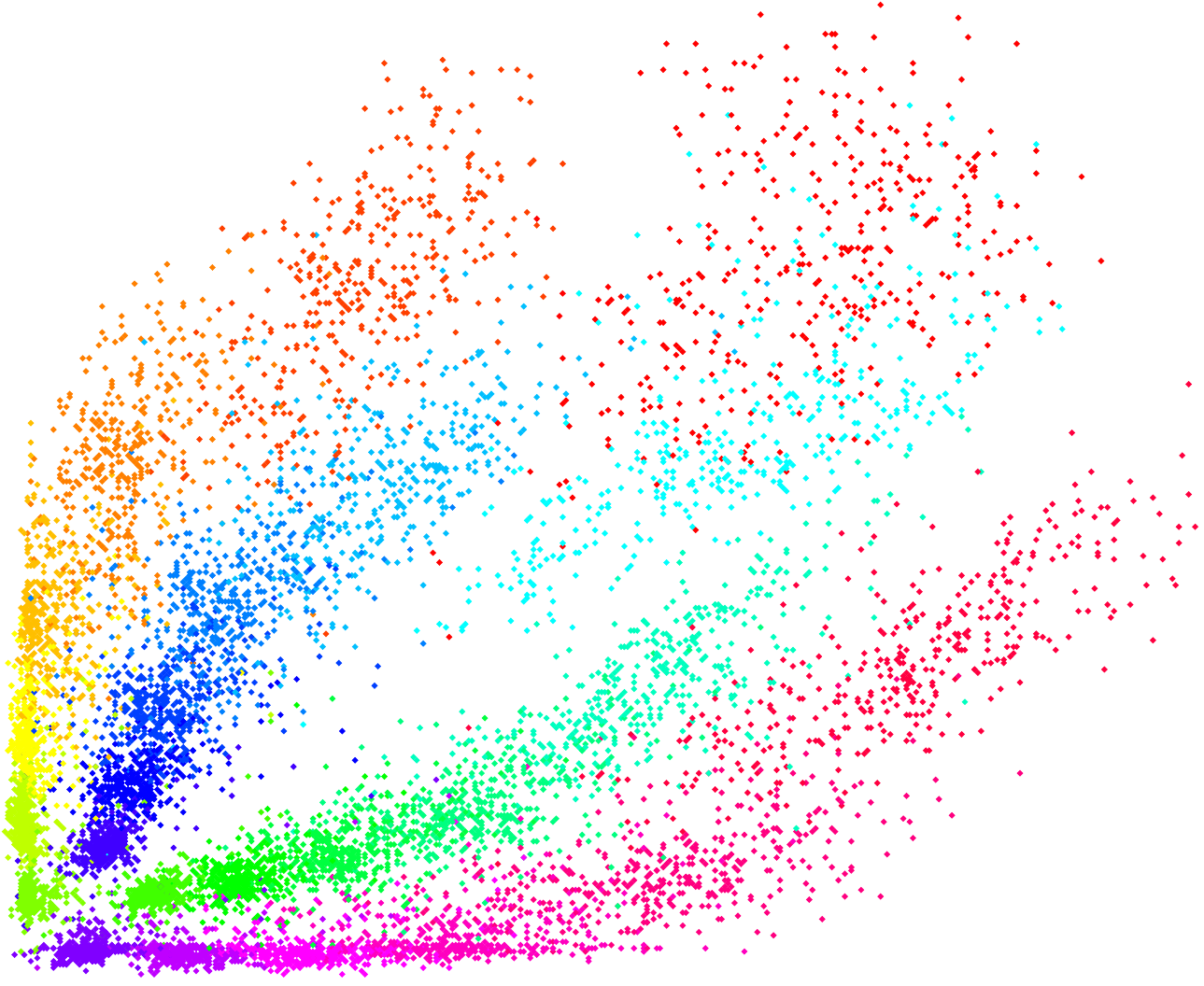}
		\caption{}
	\end{subfigure}
	\hspace{1.5cm}
	\begin{subfigure}[b]{.25\textwidth}
		\includegraphics[width=1\linewidth]{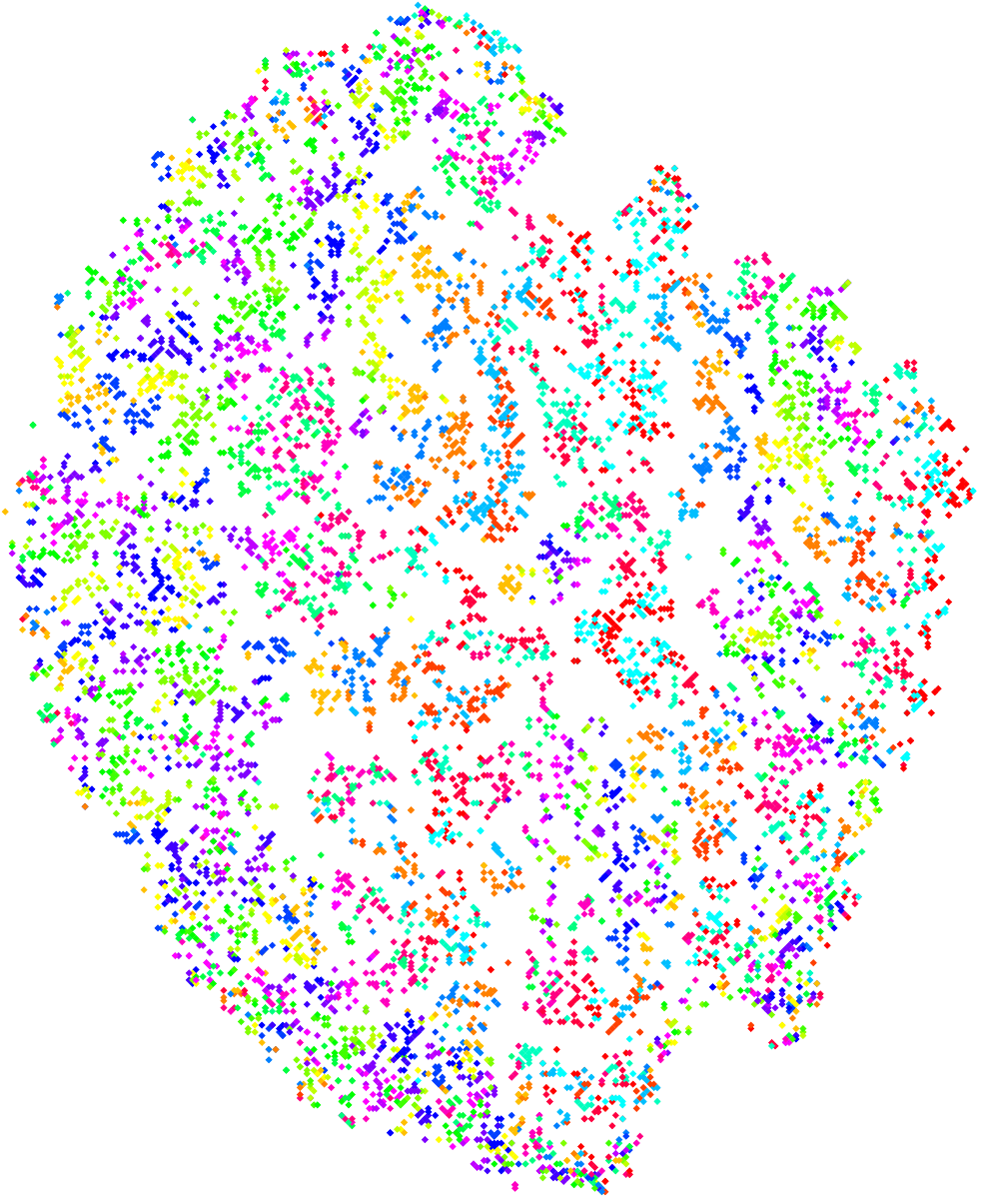}
		\caption{}
	\end{subfigure}
	\hspace{3cm}
	\begin{subfigure}[b]{.25\textwidth}
		\includegraphics[width=1\linewidth]{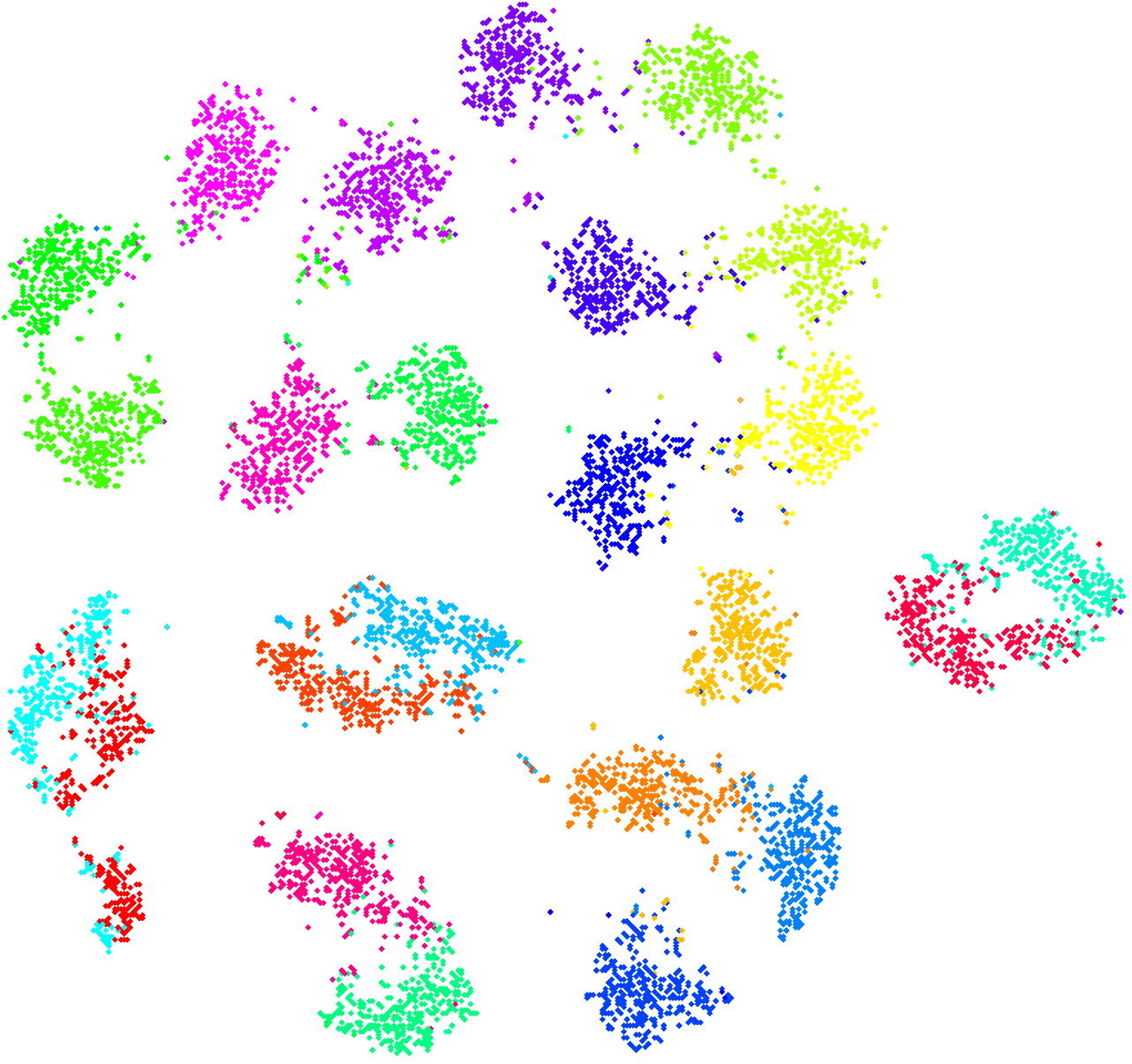}
		\caption{}
	\end{subfigure}
	\hspace{1.5cm}
	\begin{subfigure}[b]{.3\textwidth}
		\includegraphics[width=1\linewidth]{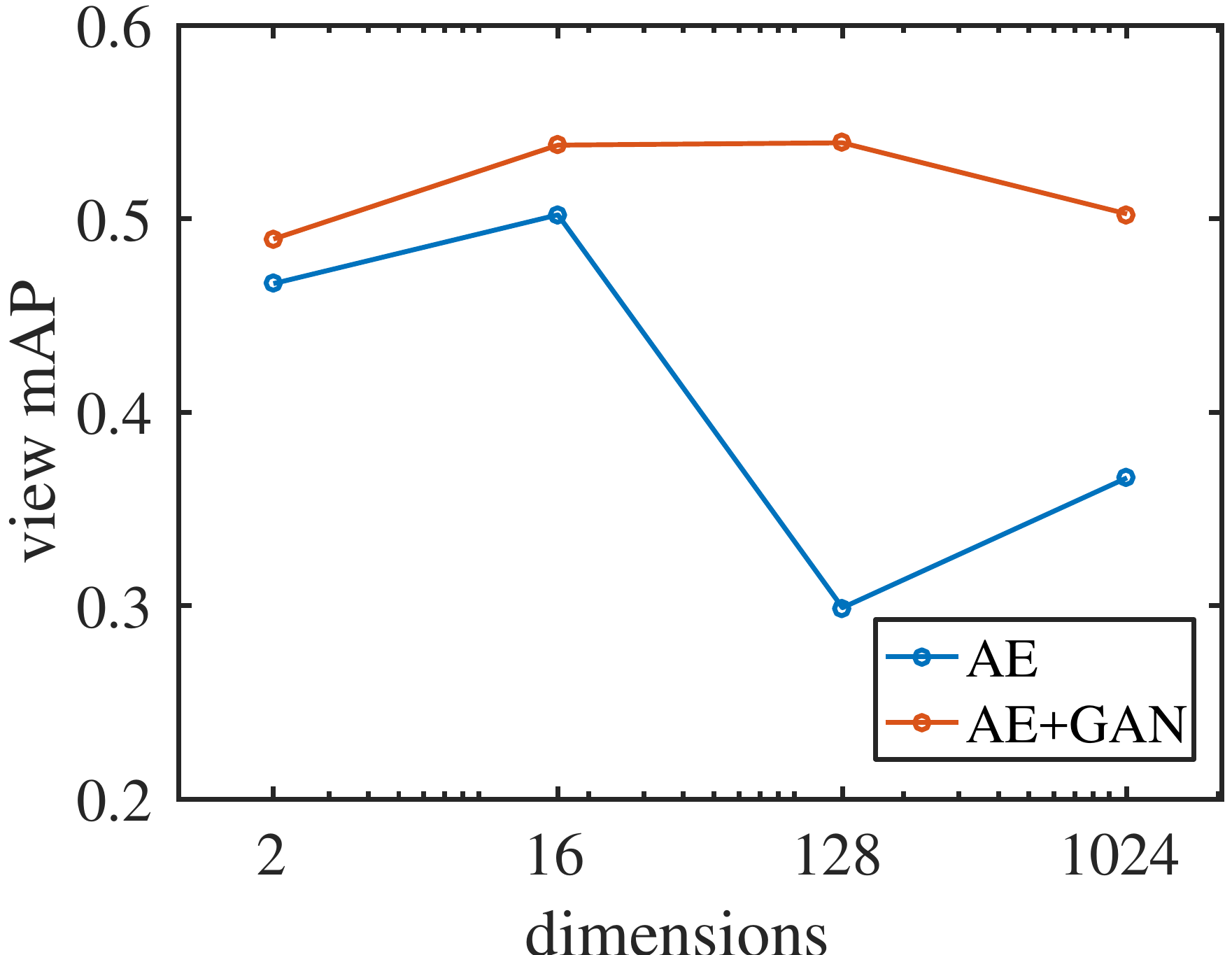}
		\caption{}
	\end{subfigure}
	\caption{The effect of dimensions and objective function on $N_v$ features. {(a)}, {(b)}, {(c)} t-SNE embeddings on $N_\v$ features. Colors correspond to the ground truth viewpoint. The objective functions and the $N_\v$ dimensions are: {(a)} \textbf{AE} $2$ dim, {(b)} \textbf{AE} $128$ dim, {(c)} \textbf{AE+GAN} $128$ dim. {(d)} Mean average precision curves for the viewpoint prediction from the viewpoint feature using different models and dimensions for $N_\v$.}
	\label{fig:visualize}
\end{figure}

In Fig.~\ref{fig:visualize} we visualize the t-SNE embeddings of the $N_\v$ features for several models using different feature sizes. For the $2D$ case, we do not modify the data. We can see that both \textbf{AE} with $2$ dimensions and \textbf{AE+GAN} with $128$ separate the viewpoints well, but \textbf{AE} with $128$ dimensions does not due to the shortcut problem.
We investigate the effect of dimensionality of the $N_\v$ features on the nearest neighbor classification task. The performance is measured by the mean average precision. For $N_\v$ we use the viewpoint as ground truth.
Fig.~\ref{fig:visualize} also shows the results on \textbf{AE} and \textbf{AE+GAN} models with different $N_\v$ feature dimensions. The dimension of $N_\c$ was fixed to $1024$ for this experiment. One can now see quantitatively that \textbf{AE} is sensitive to the size of $N_\v$, while \textbf{AE+GAN} is not. \textbf{AE+GAN} also achieves a better performance.
We used the ShapeNet with a white background dataset also to compare the different normalization choices in Table~\ref{table:BNtable}. 
We evaluate the case when batch, instance and no normalization are used and compute the performance on the nearest neighbor classification task. We fixed the feature dimensions at $1024$ for both $N_\v$ and $N_\c$ features in all normalization cases.
We can see that both batch and instance normalization perform equally well on viewpoint classification and ``no normalization'' is slightly worse. For the car type classification instance normalization is clearly better.

\begin{table}
	\caption{Nearest neighbor classification on $N_\v$ and $N_\c$ features using different normalization techniques on ShapeNet with a white background.}
	\label{table:BNtable}
	\centering
	\begin{tabular}{lll}
		\cmidrule{1-3}
		Normalization     & $N_\v$ mAP  & $N_\c$ mAP      \\
		\midrule
		\textbf{None} & 0.47 & 0.13      \\
		\textbf{Batch}    & 0.50  & 0.08      \\
		\textbf{Instance}    & 0.50  & 0.20      \\
		\bottomrule
	\end{tabular}
\end{table}


\begin{figure}[t]
	\centering
	\begin{subfigure}[b]{.45\textwidth}
		\includegraphics[width=1\linewidth,trim={0 10cm 9cm 0},clip]{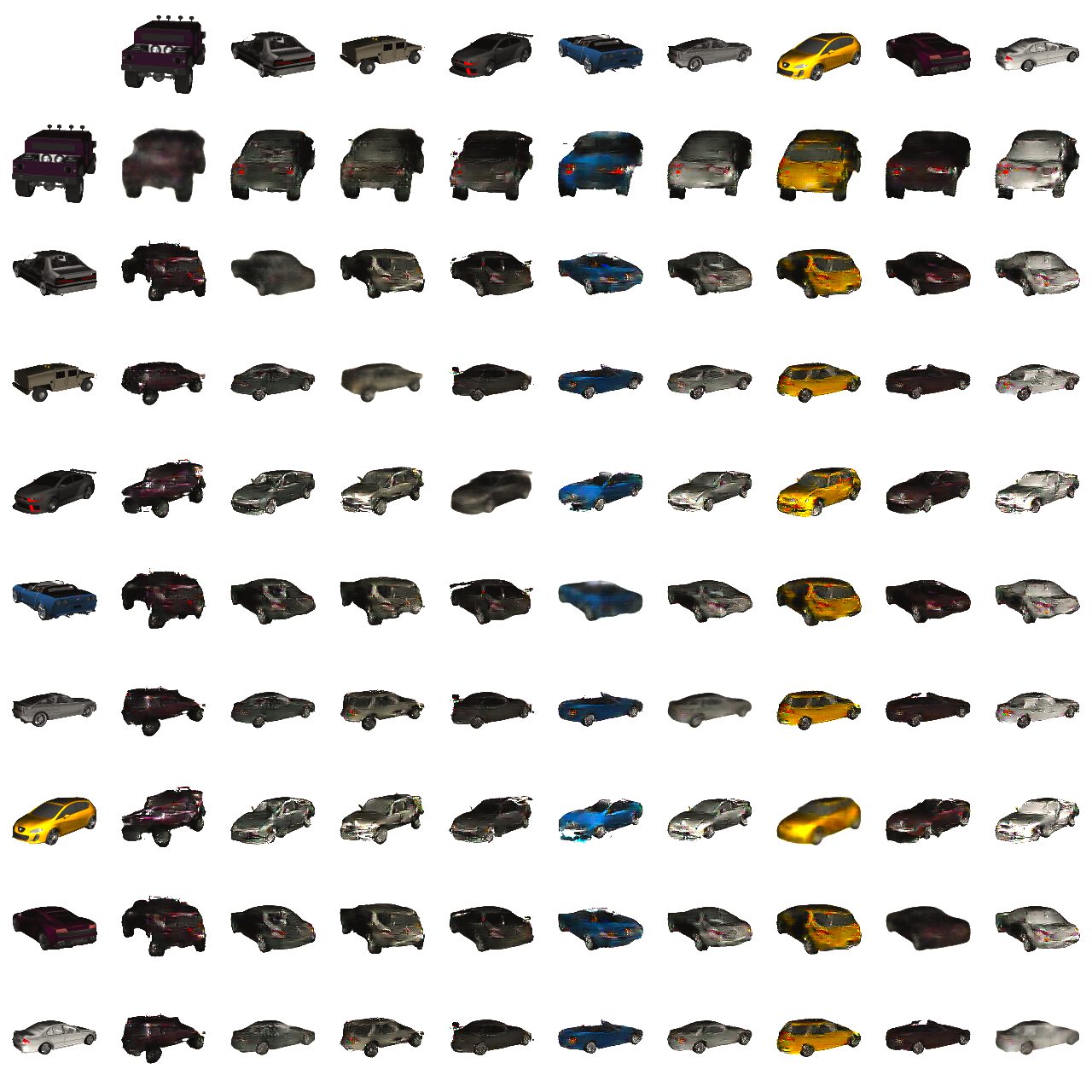}
		\caption{}
	\end{subfigure}
	\hspace{0.5cm}
	\begin{subfigure}[b]{.45\textwidth}
		\includegraphics[width=1\linewidth,trim={0 9cm 9cm 0},clip]{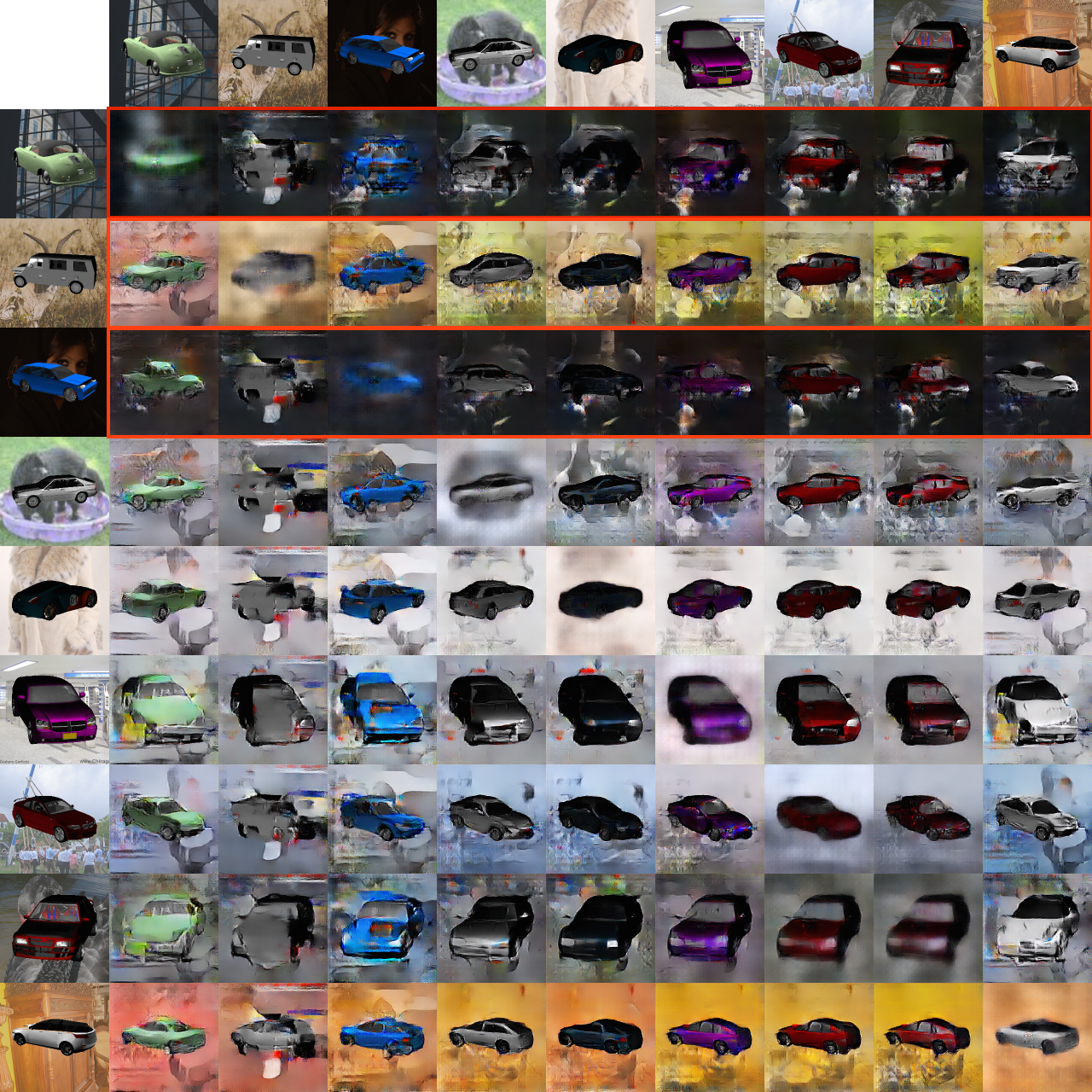}
		\caption{}
	\end{subfigure}
	\caption{ShapeNet transfers with (a) a white and (b) ImageNet background.}
	\label{fig:shapenet_background}
\end{figure}

\noindent\textbf{ShapeNet with ImageNet background. }
We render the ShapeNet dataset (same set of cars as in the previous section) with ImageNet images as background. The settings for the rendering (image size, viewpoints) are the same as in the case with a white background. We choose the backgrounds randomly for each car image, so that the overall dataset size of the data is the same, $80$K. Since the image pairs use a different background during the training, the background is also part of the varying component. In Fig.~\ref{fig:shapenet_background} we show results on attribute transfer in the case of ShapeNet with a white and with ImageNet background. We found that the reference ambiguity does not emerge in the first dataset, but it does emerge in the second dataset, possibly due to the higher complexity.
We highlight these incorrect viewpoint transfers with a red border (see top three rows in Fig.~\ref{fig:shapenet_background} (b). Nonetheless, we find that the proposed model more often than not correctly transfers the viewpoint. The background seems to transfer less well than the viewpoint, but we speculate that the background transfer might improve with better tuning and longer training.

\section{Conclusions}

In this paper we studied the challenges of disentangling factors of variation. We described the reference ambiguity and showed that it is inherently present in the task, when weak labels are used. Most importantly this ambiguity can be stated independently of the learning algorithm.
We also introduced a novel method to train models to disentangle factors of variation.
The model must be part of an autoencoder since our method requires that the representation is sufficient to reconstruct the input data.
We have shown how the shortcut problem due to feature dimensionality can be kept under control through adversarial training. We demonstrated that training and transfer of factors of variation may not be guaranteed. However, in practice we observe that our trained model works well on most datasets and exhibits good generalization capabilities.


 \small
\bibliography{refs}
\bibliographystyle{iclr2018_conference}

\end{document}